\documentclass[runningheads]{llncs}
\usepackage{beamerarticle}
\usepackage[utf8]{inputenc}
\usepackage{amsmath}
\usepackage{graphicx}
\usepackage{amssymb}
\usepackage{bm}
\usepackage[shortlabels]{enumitem}
\usepackage{graphicx}
\usepackage{csquotes}
\usepackage{float}
\usepackage{booktabs} 
\usepackage{csquotes}
\usepackage{eurosym}
\usepackage{adjustbox}
\usepackage{abstract}
\usepackage{colortbl}
\usepackage{graphicx}
\usepackage{capt-of}
\usepackage{caption}
\usepackage[sorting=none]{biblatex}
\usepackage[breaklinks=true]{hyperref}
\usepackage{mathrsfs}
\begin{document}

\title{Clustering-Based Interpretation of Deep ReLU Network}
\author{Nicola Picchiotti\inst{1, 2}\orcidID{0000-0003-3454-7250} \and Marco Gori\inst{1, 3}
}

\authorrunning{N. Picchiotti et al.}
\institute{SAILAB, University of Siena 
\url{http://sailab.diism.unisi.it} 
\and
 University of Pavia \email{nicola.picchiotti01@universitadipavia.it}\and
 MAASAI, Universit\`e C\^ote d'Azur
\email{marco@diism.unisi.it}
}

\maketitle

\begin{abstract}\label{ABS}
Amongst others, the adoption of Rectified Linear Units (ReLUs) is regarded as one of the ingredients of the success of deep learning. ReLU activation has been shown to mitigate the vanishing gradient issue, to encourage sparsity in the learned parameters, and to allow for efficient backpropagation. In this paper, we recognize that the non-linear behavior of the ReLU function gives rise to a natural clustering when the pattern of active neurons is considered. This observation helps to deepen the learning mechanism of the network; in fact, we demonstrate that, within each cluster, the network can be fully represented as an affine map. The consequence is that we are able to recover an explanation, in the form of feature importance, for the predictions done by the network to the instances belonging to the cluster. Therefore, the methodology we propose is able to increase the level of interpretability of a fully connected feedforward ReLU neural network, downstream from the fitting phase of the model, without altering the structure of the network. A simulation study and the empirical application to the Titanic dataset, show the capability of the method to bridge the gap between the algorithm optimization and the human understandability of the black box deep ReLU networks.
\end{abstract}

\section{Introduction} \label{INT}

The "black box" nature of deep neural network models is often a limit for high-stakes applications like diagnostic techniques, autonomous guide etc. since the reliability of the model predictions can be affected by the incompleteness in the optimization problem formalization \cite{lipton2018mythos}. Recent works have shown that an adequate level of interpretability could enforce the neural network trustworthiness (see \cite{doshi2017towards}). However, this is generally difficult to achieve without alter the mechanism of deep learning \cite{ribeiro2016should}. The topic has became particularly relevant in the last decades, since technology development and data availability led to the adoption of more and more complex models, and widened the gap between performances on train/test data and interpretability. 

Much work has been done in order to explain and interpret the models developed by AI in a human comprehensible manner. The main reason behind these effort is that the human experience and capacity for abstraction allow to monitor the process of the model decisions in a sound way, trying to mitigate the risks of data-driven models. On the other hand, the recent developments of deep neural networks offer enormous progress in artificial intelligence in various sectors. In particular, the ReLU functions have been shown to mitigate the vanishing gradient issue, encourage sparsity in the learned parameters and allow for efficient backpropagation.

In this paper, we tackle the explainability problem of deep ReLU networks, by characterizing the natural predisposition of the network to partition the input dataset in different clusters. The direct consequence of the clustering process is that, within each cluster, the network can be simplified and represented as an affine map. Surprisingly, the deep network provides a notion of cluster-specific importance for each feature, thus easier to interpret.

\section{Bibliographic review} \label{BIB}
Despite the benefits and the expressiveness of Rectifier Networks have been widely investigated, e.g. in \cite{pan2016expressiveness}, the cluster analysis and the consequent interpretability of the networks via the modelling of the pattern of active neurons has not been discussed in the literature, to our knowledge.

On the other hand, many model-specific methodologies can be applied for gaining interpretability in neural network models (see \cite{arrieta2020explainable} and \cite{guidotti2018survey} for a review). The importance of a feature is one of the most used strategies to gain local explainability from an opaque machine learning model. \textit{Permutation feature importance} methods evaluate the feature importance through the variation of a loss metric by permuting the feature's values on a set of instances in the training or validation set. The approach is introduced for random forest in \cite{breiman2001random} and for neural network in \cite{recknagel1997artificial}. Other methods, such as \textit{class model visualization} \cite{simonyan2014deep}, compute the partial derivative of the score function with respect to the input, and \cite{montavon2018methods} introduces expert distribution for the input giving \textit{activation maximization}. In \cite{shrikumar2017learning} the author introduces \textit{deep lift} and computes the discrete gradients with respect to a baseline point, by backpropagating the scoring difference through each unit. \textit{Integrated gradients} \cite{sundararajan2016gradients} cumulates the gradients with respect to inputs along the path from a given baseline to the instance. Finally, a set of well known methods called \textit{additive feature attribution methods} defined in \cite{lundberg2017unified}, rely on the redistribution of the predicted value $\hat{f}(\boldsymbol{x})$ over the input features. They are designed to mimic the behaviour of a predictive function $\hat{f}$ with a surrogate Boolean linear function $g$.

\section{Deep ReLU networks for the partition of the input space}\label{DEE}

In this section, we demonstrate that a deep ReLU neural network gives rise to a partition of the input dataset into a set of clusters, each one characterized by an affine map.

Let us denote by $W_{i}, \ \ i \in [1,\ldots,p]$ the weight matrices associated with the $p$ layers of a given multilayer network with predictor $\tilde{f}$ (of a $q$-dim target variable) and collect\footnote{The $\hat{}$ in the notation means that the bias term is incorporated in the variable.} $\hat{W}_{i}$ in $\hat{W} =[\hat{W}_{1},\ldots,\hat{W}_{p}]$.
For any input $\bm{u} \in \mathbb{R}^{d}$, the initial Directed Acyclic Graph (DAG) ${\cal G}$ of the deep network is reduced to ${\cal G}_{\bm{u}}$
which only keeps the units corresponding to active neurons\footnote{A neuron is considered active for a particular pattern if the input falls in the right linear part of the domain's function.} and the corresponding arcs. 
This DAG is clearly associated with a given set of weights $\hat{W}$. 
We can formally state this pruning for the given neural network, characterized by ${\cal G}$, 
paired with input $\bm{u}$ with weights $\hat{W}$ by  
\[
	{\cal G}_{\bm{u}}=\gamma({\cal G},\hat{W},\bm{u}),
\]  
where, since all neurons operate in ``linear regime'' (affine functions), as stated in the following, 
the output, the composition of affine functions, is in fact an affine function itself. 
\begin{theorem}
	Let $\mathscr{X}_{i} \subset \mathbb{R}^{d_{i}} , \ \ i \in [1,\ldots,p]$ be, where $d_{1}=d$. Let
	$\{h_1,\ldots,h_p \}$ be a collection of affine functions, where 
	$$
		h_{i}: \mathscr{X}_{i} \mapsto \mathscr{Y}_{i}: x \mapsto W_{i} \, x + \bm{b}_{i} = \hat{W}_{i} \, \hat{x},
	$$
	and assume that $\mathscr{Y}_{i}$ is chosen in such a way that
	$\forall i=1,\ldots, p-1: \ \ \mathscr{X}_{i+1} \subset \mathscr{Y}_{i}$, whereas $\mathscr{Y}_{p} \subset  \mathbb{R}^{d_{q}}$. Then we have that
	\[
		{\tilde{f}}(\hat{W},\bm{u}) = h_{p} \circ h_{p-1} \circ \dots \circ h_{2} \circ h_{1} (\bm{u})
	\]
	is affine and we have 
	$
		{\tilde{f}}(W,\bm{u})=f(\hat{\Omega},\hat{\bm{u}}) = \Omega \, \bm{u} + \bm{b},
	$ 
	where\footnote{We stress the dependence of $\Omega$ and $\bm{b}$ from the number of layer (p) since it will be useful for the proof.}
	\begin{align}
		\label{OmegaDef}
		&\hat{\Omega}:=[\Omega, \bm{b}] \\
		\label{W-eq}
		&\Omega(p) = \prod_{i=p}^{1} W_{i}\\
		\label{b-eq}
		&\bm{b}(p) =  \sum_{i=1}^{p} \left( \prod_{t=p+1}^{i+1} W_{t} \right) \cdot \bm{b}_{i} 
	\end{align}
	being $W_{p+1}:=\mathbb{1}$.
\end{theorem}
\begin{proof}
The proof is given by induction on $p$.
\begin{itemize}
\item {\em Basis}: For $p=1$ we have ${\tilde{f}}= W_1 \, \bm{u} + \bm{b}_1$ and  
	$\Omega(1)= W_1$ which confirms~(\ref{W-eq}), and when considering $W_2:=\mathbb{1}$, we have
	\[
	 	\bm{b}(1) = W_{2} \cdot \bm{b}_{1}  = \bm{b}_{1},
	\]
	in according to~(\ref{b-eq}).
\item {\em Induction step}: By induction, a network with $p-1$ layers is defined by 
	an affine transformation 
	that is 
	\[
		 y(p-1) = \Omega(p-1) \, \bm{u} + \bm{b}(p-1).
	\]
	Hence 
	\begin{align}
		y(p) &= W_{p}\,  y(p-1)+\bm{b}_{p}\nonumber\\
		&= W_{p} (\Omega(p-1) \, \bm{u} + \bm{b}(p-1))+\bm{b}_{p}\nonumber\\
		&= W_{p} \bigg(\prod_{i=p-1}^{1} W_{i} \, \bm{u} +  
		 \sum_{i=1}^{p-1}  \left( \prod_{t=p}^{i+1} W_{t} \right)  \cdot \bm{b}_{i}  \bigg) +\bm{b}_{p}\nonumber\\
		 &= W_{p} \bigg(\prod_{i=p-1}^{1} W_{i} \, \bm{u} +  
		 \sum_{i=1}^{p-2}  \left( \prod_{t=p-1}^{i+1} W_{t} \right)  \cdot \bm{b}_{i} + \mathbb{1} \cdot \bm{b}_{p-1} \bigg) +\bm{b}_{p}\nonumber\\
		&= \prod_{i=p}^{1} W_{i} \, \bm{u} +  
		 \sum_{i=1}^{p-2}  \left( \prod_{t=p}^{i+1} W_{t} \right)  \cdot \bm{b}_{i} +    W_{p} \, \bm{b}_{p-1} +\bm{b}_{p}\nonumber\\
		&= \prod_{i=p}^{1} W_{i} \, \bm{u} +  
		 \sum_{i=1}^{p-1}  \left( \prod_{t=p}^{i+1} W_{t} \right)  \cdot \bm{b}_{i} +\mathbb{1} \cdot \bm{b}_{p}\nonumber\\
		& = \bigg(\prod_{i=p}^{1} W_{i}\bigg) \, \bm{u} +
		\sum_{i=1}^{p}  \left(  \prod_{t=p+1}^{i+1} W_{t} \right) \cdot \bm{b}_{i}
	\end{align}

\end{itemize}	
\end{proof}

Now let $\mathscr{U} \subset \mathbb{R}^{d}$ the input space. The given deep net yields a partition on $\mathscr{U}$
which is associated with the following equivalence relation:
\[
	\bm{u}_1 \sim \bm{u}_{2}  \leftrightarrow 
	{\cal G}_{\bm{u}_1} = {\cal G}_{\bm{u}_2}
\].
We denote\footnote{In this work $[\cdot]$ is the Iverson's notation, whereas $[\cdot]_{\sim}$ is reserved to the equivalent class induced by equivalence relation $\sim$.} by $[\bm{u}]_{\sim} = \{\bm{v} \in \mathscr{U}:  \bm{v} \sim \bm{u}\}$ and by $\mathscr{U}/\sim$ the corresponding quotient set.
Hence, $[\bm{u}]_{\sim}$ is the equivalent class associated with representer $\bm{u}$ 
which, in turn, corresponds with ${\cal G}_{\bm{u}}$.
Notice that, as a consequence,  $[\bm{u}]_{\sim}$ is fully defined by the set of neurons  of the active
neural network. 

\paragraph{Feature Importance Explanation}
The characterization done above allows to assigning a matrix $\hat{\Omega}$ to each cluster of the network. In this way, the matrix is able to represent the network for the patterns of the specific cluster as an affine map.

For simplicity, if we consider a problem where the output of the network is scalar, the matrix $\Omega$ reduces to a d-dimensional effective vector $\boldsymbol{\omega}$ whose components can be interpreted as the \textit{feature importance of the cluster's solution}.

\section{Simulation study}\label{SIM}
In this section, we report the simulation studies carried out on the ReLU network architecture applied to a Boolean artificial dataset, in order to assess the power of the clustering-based interpretation. 

We consider a set of $10$ Boolean feature variables $v_{i\in [1, ..., 10]}$. The first $3$ features determine the target variable through the following relation:

\begin{equation}\label{target}
   t = \left( v_1 \land v_3 \right) \lor \left( v_2 \land \lnot v_3 \right) 
\end{equation}

whereas the other $7$ features introduce noise. The rationale of the formula is that the feature $v_3$ split the data set in two groups ($v_3$ and $\lnot v_3$), each ruled by a different term of the Boolean formula involving either $v_1$ or $v_2$ respectively. In the following, we investigate whether the network clustering is able to recognize the different terms of the formula.\\\\
We simulated $100,000$ samples, and we exploited a two-hidden-layer Multilayer Perceptron (MLP) with $4$ and $2$ neurons characterized by ReLU activation function and an output neuron with a sigmoid activation function. The cross-entropy loss is minimized via the Adam stochastic optimizer with a step size of $0.01$ for $10$ epochs, and a batch size of $100$. An activity regularizer with $0.02$ is added to the empirical loss. At the end of the training, the network solves the problem with an accuracy of 100\%. The experiments are implemented with Keras in the Python environment on a regular CPU.

The analysis of the network, by considering the possible patterns of the active neurons, originates three clusters. 
\begin{enumerate}
\item A trivial cluster characterized by all non-active neurons including all the patterns predicted as $0$.
\end{enumerate}
Instead, the other two clusters activate the two neurons of the last layer but a different neuron of the first hidden layer. In figure~\ref{simul} we report the table resuming the $8$ possibilities of the Boolean function restricted to the first $3$ features, as well as the bar plot for the importance of the features for the two non-trivial clusters. As explained above, the importance of the feature is computed by the coefficients of the effective vector for that cluster.
\begin{enumerate}
	\setcounter{enumi}{1}
\item The second cluster, represented by the blue bars in Fig.~\ref{simul}, includes patterns predicted as $1$ and characterized either by $v_1 =v_2=v_3=1$ or by $v_1 =1 ,\, v_2=0 ,\, v_3=1$. We can argue that this cluster takes in charge the patterns predicted as $1$ due to the first term of Eq.~\eqref{target}, i.e. $v_1 \land v_3$. 
Coherently with this setting, the feature importance given by the effective vector is zero for the feature $v_2$ that does not appear in the first term of Eq.~\eqref{target}. On the other side, the coefficient of the effective vector is positive for the feature $v_1$.
\item In a similar way, the third cluster denoted with the orange color, represents the term $v_2 \land \lnot v_3$ of Eq.~\eqref{target} since the patterns belonging to it are predicted as $1$ and are characterized either by $v_1 =v_2=1 \, v_3=0$ or by $v_1 =0\,, v_2=1 \,, v_3=0$. As expected, the feature importance of $v_1$ is zero, whereas for $v_2$ the feature importance is positive.
\end{enumerate}

\begin{figure}[!htb]
\centering
  \includegraphics[width=1\linewidth]{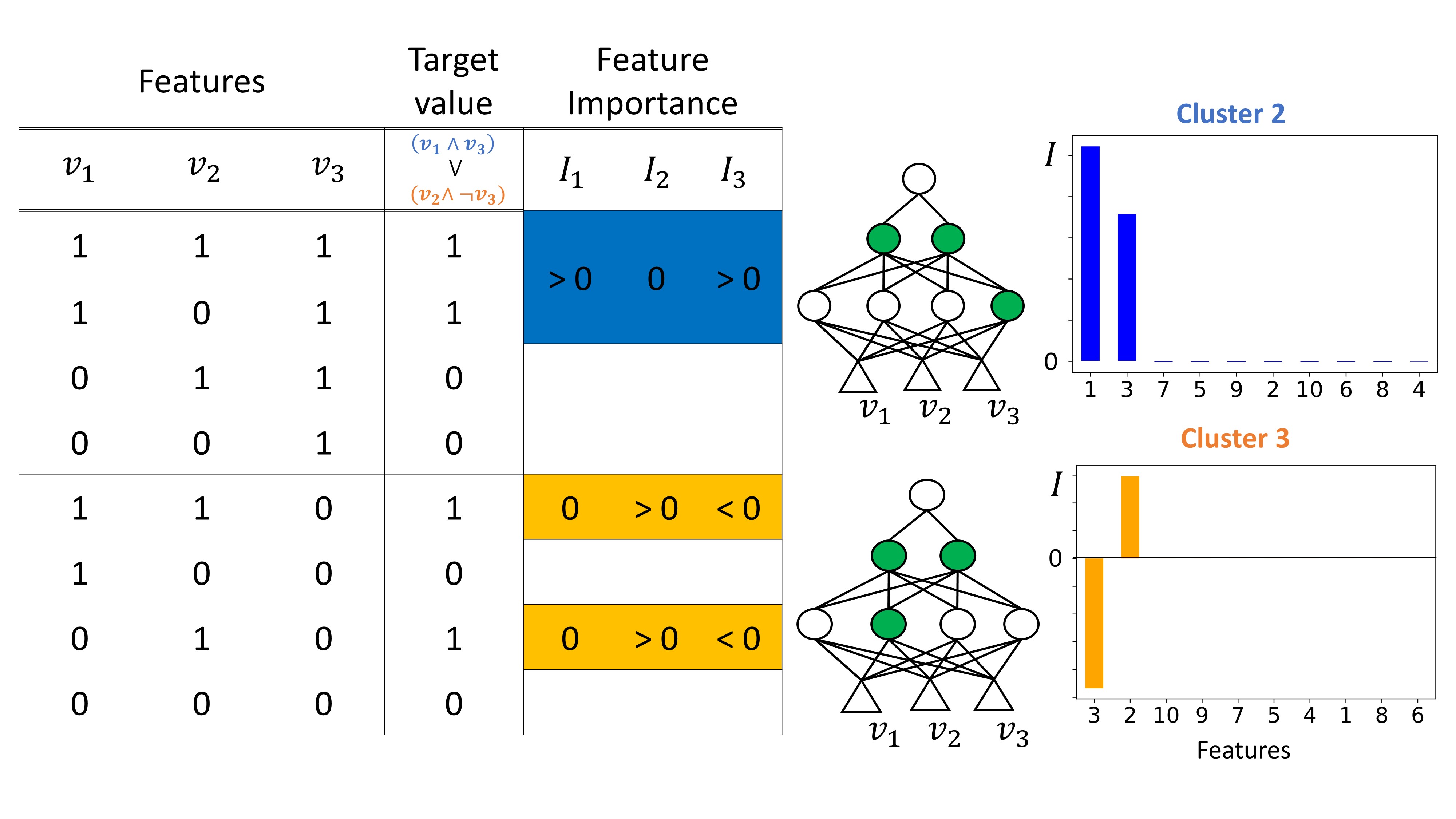}
  \caption{Possible results of the simulated task (8 combinations) and representation of the actual clusters provided by the cluster-based interpretation of the ReLU network.}
  \label{simul}
\end{figure}

From the simulation study, we note that the clustering-based interpretation of the ReLU network helps to achieve a more profound understanding of the solution meaning. In particular, the methodology tries to disentangle the complexity of the solutions into a set of more comprehensible linear solutions within a specific cluster. 

As a further confirmation of the usefulness of cluster-based interpretation, when the activity regularizer is removed, the network keeps solving the task giving rise to $8$ clusters, each one specific for each combination of the first three features. Based on the cluster interpretation, we conclude that the network has chosen a less abstract way to solve the problem.

\section{Titanic dataset}

In this section, we report the experimental analysis performed on the well-known Titanic dataset\footnote{\url{https://www.kaggle.com/c/titanic}}. Each sample represents a passenger with specific features, and the binary target variable indicates if the person survived the Titanic disaster. Standard data cleaning and feature selection procedures\footnote{ \url{https://www.kaggle.com/startupsci/titanic-data-science-solutions}} are implemented. In Table~\ref{titanic} we report a brief description of the features.

\begin{table}[!htb]
  \caption{Features for the Titanic dataset}
  \label{titanic}
  \centering
  \begin{tabular}{ccc}
    \toprule
    Feature     & Description     & Range \\
    \midrule
    Age & Age of the passenger discretized in 5 bins & $[0,4]$     \\
    Gender     & 1 if the passenger is female  & $\{0, 1\}$      \\
    PClass     &  Travel class: first, second, third       & $[1, 3]$  \\
    Fare     & Ticket fare discretized in 3 bins    & $[0, 3]$ \\
    Embarked     & Location for the embarked  & $[0, 2]$  \\
    Title     & Title of the passenger: Mr, Miss, Mrs, Master, Rare    & $[1, 5]$  \\
    Is Alone     & 1 if the passenger has not relatives & $\{0, 1\}$
  \end{tabular}
\end{table}

Similar to the previous experiment, we exploit an MLP with two hidden layers ($4$ and $2$ neurons) characterized by ReLU activation function and an output neuron with a sigmoid activation function. The cross-entropy loss is minimized via the Adam stochastic optimizer with a step size of $0.01$ for $10$ epochs and a batch size of $100$. An activity regularizer with $0.02$ is added to the empirical loss. The experiments are implemented with Keras in the Python environment. The code is freely available at \url{https://github.com/nicolapicchiotti/relu\_nn\_clustering}.

The accuracy of the network is $77\%$ (see \cite{Lamcs229titanic}) and the study of the active neurons patterns provides a partition of the dataset into three clusters, whose feature importance is shown in Figure~\ref{figt}.
\begin{enumerate}[(a)]
\item The first cluster a) includes passengers with mixed features and a percentage of survived ones equal to $38\%$. As expected from the univariate exploratory analyses, \textit{gender} and \textit{class} had the most significant relationship for survival rate.
\item The cluster b) instead, includes only males belonging to the third class ($4\%$ of the overall): the prediction for these passengers is always $0$. This cluster confirms the expectation on the male gender and third class as relevant features for not survive. We observe that the feature importance is quite similar to the one of cluster a) except for the fact that the "age" feature assumes slightly more relevance. Finally, 
\item in the cluster c) the passengers are females and belonging to the first class ($16\%$ of the overall),  the predicted value is always 1. The feature importance, in this case, shows that the high value of the \textit{title}, \textit{age}, and the other features contribute to survival, in addition to being female. This cluster helps us to understand the solution provided by the network. For instance, we note that the "age" feature has an opposite behavior with respect to the other two clusters: in this cluster the older the women, the higher the survival probability.
\end{enumerate}

In this experiment, we have shown that the ReLU network can be disentangled into a set of clusters that can be analyzed individually. The clusters have a practical meaning helping the human to interpret the mechanism of prediction of the network.

\begin{figure}[!htb]
\centering
  \includegraphics[width=0.8\linewidth]{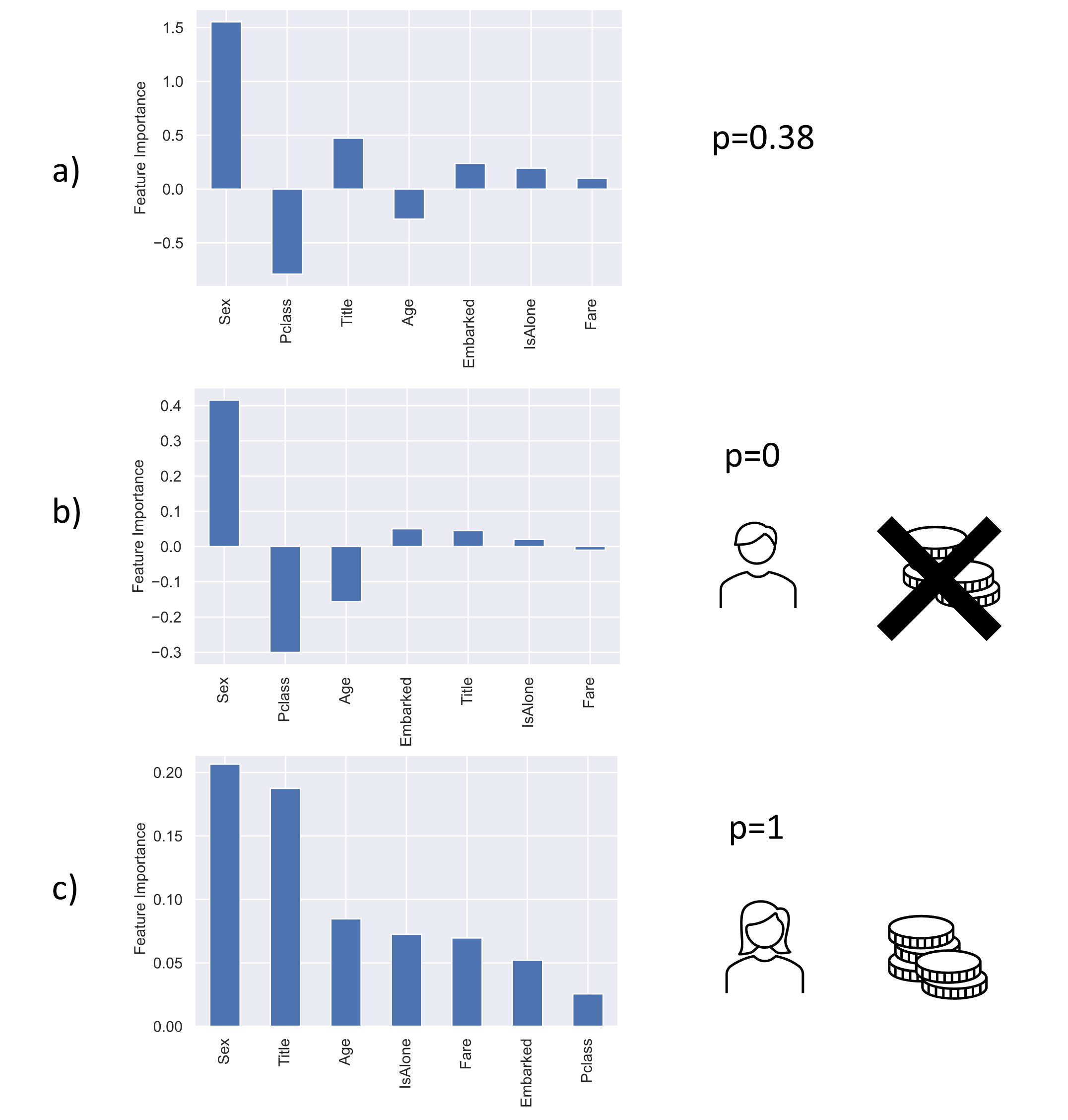}
  \caption{Bar plot with the feature importance for each of the three clusters originated by the ReLU neural network.}
  \label{figt}
\end{figure}

\section{Conclusions and Limitations}\label{CON}

This paper proposes a methodology for increasing the level of interpretability of a fully connected feedforward neural network with ReLU activation functions, downstream from the fitting phase of the model. It is worth noting that the methodology does not alter neither the structure nor the performance of the network and can be applied easily after the training of the model. The methodology relies on the clustering that naturally arises from the binary status of the different neurons of the network, in turns, related to the non-linearity of the ReLU function. The existence of a feature importance explanation based on a affine map for each cluster is proved.
A simulation study and the empirical application to the Titanic dataset show the capability of the method to bridge the gap between the algorithm optimization and the human understandability.

A possible limitation of the work is related to the potential high number of clusters that the networks could generates. Further ways should be explored in order to grant a parsimonious principle of the dataset partition process, such as a principle of minimum entropy or the orthogonality of the hyperplanes representing the clusters. Finally, it is worth mentioning that, if the input variables are highly correlated or belong to a very high-dimensional spaces, the weights of the affine map cannot be trivially interpreted as the features importance, see for instance \cite{franay2014valid}.
\printbibliography

\end{document}